\newif\ifisarxiv
\isarxivtrue

\ifisarxiv
\documentclass[12pt]{sty/colt2019/colt2018-arxiv}
\else
\documentclass[anon,12pt]{sty/colt2019/colt2019} 
\fi

\title[Minimax experimental design]{Minimax experimental design:
  Bridging the gap between statistical and worst-case approaches to
  least squares regression}
\usepackage{times}

\usepackage{color}
\usepackage{cancel}
\usepackage{wrapfig}
\usepackage{caption}
\usepackage{algorithm}
\usepackage{xfrac}
\usepackage{algorithmic}

\newcommand{\MSPE}[1] {{\mathrm{MSPE}\big[#1\big]}}
\newcommand{\MSE}[1] {{\mathrm{MSE}\big[#1\big]}}

\newif\ifDRAFT
\DRAFTtrue
\ifDRAFT
\newcommand{\marrow}{\marginpar[\hfill$\longrightarrow$]{$\longleftarrow$}}
\newcommand{\niceremark}[3]
   {\textcolor{red}{\textsc{#1 #2:} \marrow\textsf{#3}}}
\newcommand{\ken}[2][says]{\niceremark{Ken}{#1}{#2}}

\newcommand{\michael}[2][says]{\niceremark{Michael}{#1}{#2}}
\newcommand{\michal}[2][says]{\niceremark{Michal}{#1}{#2}}
\else
\newcommand{\ken}[1]{}
\newcommand{\michael}[1]{}
\newcommand{\michal}[1]{}
\fi
\newcommand{\norm}[1]{{\| #1 \|}}

\def\simiid{\overset{\textnormal{\fontsize{6}{6}\selectfont
i.i.d.}}{\sim}}

\def\wols{\w_{\mathrm{LS}}}

\def\Vol{{\mathrm{VS}}}

\newenvironment{proofof}[2]{\par\vspace{2mm}\noindent\textbf{Proof of {#1} {#2}}\ }{\hfill\BlackBox\\[2mm]}

\def\xib{\boldsymbol\xi}

\def\S{\mathbf{S}}

\def\val {{\mathrm{val}}}

\def\Nc{\mathcal{N}}

\def\Wc{\mathcal{W}}

\ifx\BlackBox\undefined
\newcommand{\BlackBox}{\rule{1.5ex}{1.5ex}}  
\fi
\DeclareMathOperator*{\argmin}{\mathop{\mathrm{argmin}}}

\def\x{\mathbf x}
\def\y{\mathbf y}

\def\z{\mathbf z}
\def\a{\mathbf a}
\def\b{\mathbf b}
\def\w{\mathbf w}
\def\v{\mathbf v}

\def\wbh{\widehat{\mathbf w}}

\def\vbh{\widehat{\mathbf v}}

\def\e{\mathbf e}
\def\zero{\mathbf 0}
\def\one{\mathbf 1}
\def\u{\mathbf u}

\def\X{\mathbf X}

\def\A{\mathbf A}

\def\U{\mathbf U}

\def\D{\mathbf D}

\def\M{\mathbf M}

\def\Fc{\mathcal{F}}
\def\Vc{\mathcal{V}}

\def\I{\mathbf I}

\def\A{\mathbf A}
\def\P{\mathbf P}

\def\E{\mathbb E}
\def\R{\mathbb R} 
\def\Pr{\mathrm{Pr}} 
\def\tr{\mathrm{tr}}

\def\rank{\mathrm{rank}}

\def\Var{\mathrm{Var}}

\newcommand{\defeq}{\stackrel{\textit{\tiny{def}}}{=}}

\newcommand{\cov}{\mathrm{cov}}
\let\origtop\top
\renewcommand\top{{\scriptscriptstyle{\origtop}}} 

\definecolor{silver}{cmyk}{0,0,0,0.3}
\definecolor{yellow}{cmyk}{0,0,0.9,0.0}
\definecolor{reddishyellow}{cmyk}{0,0.22,1.0,0.0}
\definecolor{black}{cmyk}{0,0,0.0,1.0}
\definecolor{darkYellow}{cmyk}{0.2,0.4,1.0,0}
\definecolor{orange}{cmyk}{0.0,0.7,0.9,0}
\definecolor{darkSilver}{cmyk}{0,0,0,0.1}
\definecolor{grey}{cmyk}{0,0,0,0.5}
\definecolor{darkgreen}{cmyk}{0.6,0,0.8,0}

\ifx\proof\undefined
\newenvironment{proof}{\par\noindent{\bf Proof\ }}{\hfill\BlackBox\\[2mm]}
\fi

\ifx\theorem\undefined
\newtheorem{theorem}{Theorem}
\fi

\ifx\example\undefined
\newtheorem{example}{Example}
\fi

\ifx\condition\undefined
\newtheorem{condition}{Condition}
\fi
\ifx\property\undefined
\newtheorem{property}{Property}
\fi

\ifx\lemma\undefined
\newtheorem{lemma}[theorem]{Lemma}
\fi

\ifx\proposition\undefined
\newtheorem{proposition}[theorem]{Proposition}
\fi

\ifx\remark\undefined
\newtheorem{remark}[theorem]{Remark}
\fi

\ifx\corollary\undefined
\newtheorem{corollary}[theorem]{Corollary}
\fi

\ifx\definition\undefined
\newtheorem{definition}{Definition}
\fi

\ifx\conjecture\undefined
\newtheorem{conjecture}[theorem]{Conjecture}
\fi

\ifx\axiom\undefined
\newtheorem{axiom}[theorem]{Axiom}
\fi

\ifx\claim\undefined
\newtheorem{claim}[theorem]{Claim}
\fi

\ifx\assumption\undefined
\newtheorem{assumption}[theorem]{Assumption}
\fi

\ifx\condition\undefined
\newtheorem{condition}[theorem]{Condition}
\fi

\coltauthor{\Name{Micha{\l } Derezi\'{n}ski}
  \Email{mderezin@berkeley.edu}\\
  \addr Department of Statistics, UC Berkeley 
  \AND
  \Name{Kenneth L. Clarkson}
  \Email{klclarks@us.ibm.com}\\
  \addr IBM Research -- Almaden
  \AND
  \Name{Michael W. Mahoney}
  \Email{mmahoney@stat.berkeley.edu}\\
  \addr ICSI and Department of Statistics, UC Berkeley
  \AND
\Name{Manfred K. Warmuth} \Email{manfred@ucsc.edu}\\
\addr Google Inc. Z\"urich \& UC Santa Cruz
} 

\begin{document}

\maketitle

\begin{abstract}%
In experimental design, we are given a large collection of vectors,
each with a hidden response value that we assume derives from an
underlying linear model, and we wish to pick a small subset of the
vectors such that querying the corresponding responses will lead to a
good estimator of the model.  
A classical approach in statistics is to assume the responses are
linear, plus zero-mean i.i.d.~Gaussian noise, in which case the goal
is to provide an unbiased estimator with smallest mean squared error
(A-optimal design).  
A related approach, more common in computer science, is to assume the
responses are arbitrary but fixed, in which case the goal is to
estimate the least squares solution using few responses, as quickly as
possible, for worst-case inputs.   
Despite many attempts, characterizing the relationship between these
two approaches has proven elusive.  
We address this by proposing a framework for experimental 
design where the responses are produced by
an arbitrary unknown distribution.  
We show that there is an efficient randomized experimental design
procedure that achieves strong variance bounds for an
unbiased estimator using few responses  in this general model.
Nearly tight bounds for the classical A-optimality criterion, as well
as improved bounds for worst-case responses, emerge as special cases
of this result.  
In the process, we develop a new algorithm for a joint sampling
distribution called volume sampling, and we propose a new
i.i.d.~importance sampling method: inverse score sampling.  
A key novelty of our analysis is in developing new expected error
bounds for worst-case regression by controlling the tail behavior of
i.i.d.~sampling via the jointness of volume sampling.  
Our result motivates a new minimax-optimality criterion for
experimental design which can be viewed as an extension of both
A-optimal design and sampling for worst-case regression. 
\end{abstract}


\section{Introduction}
Consider fixed design regression in $d$ dimensions, with $n\gg d$
experiments parameterized by vectors $\x_1,\dots,\x_n\in\R^d$ and the
associated real random response variables $y_1,\dots,y_n$. Suppose
that each response variable is modeled as a linear function of the
parameters plus i.i.d.~Gaussian noise:
$y_i=\x_i^\top\w^*+\xi_i$, where $\xi_i\sim\Nc(0,\sigma^2)$. 
Let $\X$ be the $n\times d$ matrix whose rows are $\x_i^\top$ 
(assumed to be full rank)
and let $\y$ be the vector of the $n$ random responses
$y_i$. Under the above standard statistical assumptions the \textit{least squares estimator}
$\wols(\y|\X)=\X^\dagger \y$ (where $\X^\dagger=(\X^\top\X)^{-1}\X^\top$ is
the Moore-Penrose pseudo-inverse) is known to be the minimum
variance unbiased estimator for $\w^*\in\R^d$.
This implies that it satisfies $\E_\y[\wols(\y|\X)]=\w^*$, while achieving the smallest possible \textit{mean squared error}:
$\text{MSE}[\wols(\y|\X)]=\E_\y\big[\|\wols(\y|\X)-\w^*\|^2\big]=\sigma^2\phi$, where
$\sigma^2$ is the magnitude of the noise and $\phi=\tr((\X^\top\X)^{-1})$ captures the relevant spectral structure of $\X$. 
To compute this estimator exactly, we have to observe all $n$ responses. 

In the realm of experimental design \citep{optimal-design-book}, one asks: what if we are given all
$n$ vectors $\x_i$ but are allowed to query only $k\ll n$ of the
responses? An unbiased estimator produced under this additional
restriction will certainly be no better than $\wols(\y|\X)$ (in terms of its MSE). 
There are many experimental design criteria that have been considered.
For example, we may wish to find a weight vector that minimizes the excess mean squared error resulting from the restricted access to the responses. 
This criterion is known as an A-optimal design. 
In this model, the problem reduces to finding a
subset $S\subseteq [n]$ of $k$ experiments for which the mean squared
error of the least squares estimator is minimized. Its MSE then becomes
$\min_S\sigma^2\tr((\X_S^\top\X_S)^{-1})$, where $\X_S$ is a submatrix
with $k$ rows selected by $S$.
Other optimality criteria have been studied for selecting subset $S$,
e.g., V-optimality (which we discuss below), as well as D- and
E-optimality (which are not based on the variance of the estimator,
therefore they are not as relevant to this discussion).   

How good (in terms of the MSE) can the A-optimal subset be in general?
Not surprisingly, this will depend on the total noise of the responses,
i.e.~$\E\big[\|\xib\|^2\big]=n\sigma^2$, where $\xib\in\R^n$ is the vector of
noise variables $\xi_1,\dots,\xi_n$, as well as the structure of $\X$ described
by $\phi=\tr((\X^\top\X)^{-1})$.
The following result from numerical linear algebra 
shows the existence of a subset $S$ with a
good A-optimality bound as a function of its size $k$ 
which is known to be asymptotically tight for some
matrices. The resulting experimental design given in
the corollary can be computed efficiently. 
\begin{theorem}[\citeauthor{avron-boutsidis13}, \citeyear{avron-boutsidis13}]
\label{thm:avron-boutsidis}
For any full rank $\X\in\R^{n\times d}$ and $d\leq k\leq n$, there is a
subset $S\subseteq[n]$ of size $k$ s.t.~$\tr((\X_S^\top\X_S)^{-1})\leq
\frac{n-d+1}{k-d+1}\tr((\X^\top\X)^{-1})$. 
\end{theorem}
Although Theorem \ref{thm:avron-boutsidis}
was originally stated as a worst-case linear algebra statement, it
easily leads to the following corollary regarding the statistical MSE.
Here $\y_S$ denotes the vector of the selected random responses.
\begin{corollary}\label{c:classic}
Given $\X\in\R^{n\times d}$ such that 
$\tr((\X^\top\X)^{-1})\!=\!\phi$ and $\epsilon>0$, there is an
experimental design $S\subseteq [n]$
of size $k\leq d+\phi/\epsilon$ s.t.~for any $\y=\X\w^*+\xib$, where $\E[\xib] = \zero$
and $\Var[\xib]=\sigma^2\I$,
\begin{align*}
  \E_\y\big[\wols(\y_S|\X_S)\big]=\w^*\quad\text{and}\quad
  \underbrace{\mathrm{MSE}\big[\wols(\y_S|\X_S)\big]}_{\leq\frac{n-d+1}{k-d+1}\sigma^2\phi} -
    \underbrace{\mathrm{MSE}\big[\wols(\y|\X)\big]}_{\sigma^2\phi} \leq \epsilon\cdot
  \underbrace{\E_\y\big[\|\xib\|^2\big]}_{n\sigma^2}.
\end{align*}
\end{corollary}
Note that the bound in Corollary \ref{c:classic} holds even without
subtracting $\MSE{\wols(\y|\X)}$, however we include it here for the sake of
consistency with the later discussion.

\subsection{Experimental design with arbitrary random responses}
While noise $\xib$ need not be i.i.d.~Gaussian to show
Corollary~\ref{c:classic}, it still has to be
zero-mean, homoscedastic (same variances) and uncorrelated.
In this section we show that there are experimental designs for
which the MSE bound from Corollary \ref{c:classic} holds
for any (even adversarial) noise. This will allow us to propose a new
``minimax-optimality'' criterion for experimental design (in Section
\ref{s:minimax}) which can be viewed as a generalization of
A-optimality to arbitrary random responses. From now on, the only
assumption we make on the random variables $y_i$ is that they have a
finite second moment. We next redefine the optimal linear predictor $\w^*$
and the vector of noise variables $\xib$ as:\footnote{Using the fact
  that $\E_\y \big[\norm{\X\w-\y}^2\big]=\E_\y
  \big[\norm{\X\w-\E[\y]}^2\big] + \E\big[\norm{\y - \E[\y]}^2\big]$.} 
\begin{align*}
  \w^*\defeq  \argmin_\w \E_\y \big[\|\X\w-\y\|^2\big]
  	=\X^\dagger\E[\y],\qquad\xib_{\y|\X}\defeq \X\w^*-\y.
\end{align*}
Note that when the noise
  happens to have mean zero, i.e.~$\E[\xib_{\y|\X}]=\zero$, then this definition
  of $\w^*$ is consistent with the statistical setting. Having no
  knowledge of the response model means that we cannot commit to a
  particular fixed subset $S$ because those responses could be
  adversarially noisy. To avoid this, we allow randomization in the design
  procedure.
  \begin{definition}
      \label{d:rand}
A ``random experimental design'' $(S,\wbh)$ of
size $k$ consists of a \textbf{random} set $S\subseteq [n]$ of size at most $k$
and a \textbf{random} function $\wbh:\R^{|S|}\rightarrow \R^d$, which
returns an estimator $\wbh(\y_S)$.
\end{definition}
The mean squared error in this context is defined as:
  $\text{MSE}\big[\wbh(\y_S)\big] =
  \E_{S,\wbh,\y}\big[\|\wbh(\y_S)-\w^*\|^2\big]$,
so it is exactly the standard MSE, except with the expectation taken
over the randomness of both the responses and the design. 
Our main result shows that when we allow the experimental design procedure to be randomized,
the mean squared error bound
given in Corollary \ref{c:classic} for homoscedastic noise can be
recovered almost exactly for arbitrary random response vectors $\y$ 
(which includes deterministically chosen response vectors as a special case).
\begin{theorem}\label{t:mse}
Given $\X\in\R^{n\times d}$ such that 
$\tr((\X^\top\X)^{-1})\!=\!\phi$ and $\epsilon>0$, there is a
random experimental design $(S,\wbh)$
of size $k=O(d\log n+\phi/\epsilon)$ s.t.~for \textbf{any} random
response vector $\y$, 
\begin{align*}
\E_{S,\wbh,\y}\big[\wbh(\y_S)\big]=\w^*\quad\text{and}\quad
  \mathrm{MSE}\big[\wbh(\y_S)\big] - \mathrm{MSE}\big[\wols(\y|\X)\big]
  \leq \epsilon\cdot 
  \E_\y\big[\|\xib_{\y|\X}\|^2\big].
\end{align*}
\end{theorem}
To put this result in context we consider several different response
models to which it applies:
\begin{enumerate}
\item \textit{A-optimal experimental design}. If we assume
independent homoscedastic zero-mean 
noise, then
  our model matches the classical A-optimal
  experimental design, except for allowing the design procedure to be
  randomized. Despite the broadness of Theorem~\ref{t:mse} it still
  offers sample complexity that is only a log factor away from that of
  Corollary~\ref{c:classic}. 
  \item \textit{Heteroscedastic regression}. We let each
    response have zero-mean noise with some unknown variance
    $\Var[\xi_i]=\sigma_i^2$. In this case, unlike existing work such as
    \cite{regularized-volume-sampling}, we bound the MSE in terms of
    $\sum_i\sigma_i^2$ rather than $n\cdot\max_i\sigma_i^2$. Our
    design achieves this without having to adaptively estimate the
    variances as done by \cite{v-optimal}.
  \item \textit{Bayesian regression}. Suppose that
    $\y=\X\w+\z$, where $\w\sim D_\w$ is a random vector 
with a prior $D_\w$ and mean $\w^*$, whereas $\z$ is a zero-mean
random noise. In this case we may wish to minimize MSE w.r.t.~$\w$ (and not
$\w^*$),
i.e., $\E_{S,\wbh,\w,\z}\big[\|\wbh(\y_S)-\w\|^2\big]$.
For this purpose we can still apply Theorem \ref{t:mse} to the response
vector $\y$ conditioned on $\w$, obtaining:
    \begin{align*}
      \underbrace{\E\big[\|\wbh(\y_S)-\w\|^2\big]
      \ -\ \E\big[\|\wols(\y|\X)-\w\|^2\big]}_{
      \E\big[\E\big[\|\wbh(\y_S)-\w\|^2-\|\wols(\y|\X)-\w\|^2\ |\ \w \big]\big]} 
      \ \leq\hspace{-2mm}\underbrace{\epsilon\cdot\tr\big(\Var[\z]\big)}_{\E\big[\epsilon\,\cdot\,\E[\|\X\w-\y\|^2\,|\w]\big]}\hspace{-4mm}.
    \end{align*}
    While traditional Bayesian experimental design
    \citep[see][]{bayesian-design-review} focuses on 
    i.i.d.~Gaussian noise, our results apply to arbitrary
    zero-mean noise. A natural future direction is to extend Theorem \ref{t:mse} to
    biased estimators that take advantage of the prior information.
  \item \textit{Worst-case regression}. We let $\y$ be some arbitrary
    fixed vector $\y\in\R^n$, i.e., $\Var[\xib_{\y|\X}]=\zero$ \citep[a well-studied
    problem; see, e.g.,][]{drineas2006sampling}. Then $\wols(\y|\X)=\w^*$
    and we get: 
    \begin{align*}
      \E_{S,\wbh}\big[\|\wbh(\y_S)-\w^*\|^2\big] \leq \epsilon\cdot \|\X\w^*-\y\|^2,
    \end{align*}
the first such bound that holds: (a) ``in
    expectation'' (rather than ``with constant probability''), (b) for an unbiased
    estimator, (c) for sample size $O(\phi/\epsilon)$ (when $\epsilon$
    is sufficiently small).
\end{enumerate}
As a corollary to Theorem \ref{t:mse}, we give an additional result
which bounds the \textit{mean squared prediction error} (MSPE) instead
of MSE, defined as   $\text{MSPE}\big[\wbh(\y_S)\big] =
\E_{S,\wbh,\y}\big[\|\X(\wbh(\y_S)-\w^*)\|^2\big]$.
In many tasks, the 
performance of an estimator is evaluated 
in terms of the prediction accuracy, in which case MSPE may be a
natural metric. Note that here the sample complexity no longer depends
on the spectral parameter $\phi$ (which is replaced by $d$), just as
it happens when bounding MSPE in the classical homoscedastic setting.
\begin{theorem}\label{t:mspe}
Given a full rank $\X\in\R^{n\times d}$ and $\epsilon>0$, there is a
random experimental design $(S,\wbh)$
of size $k=O(d\log n+d/\epsilon)$ such that for \textbf{any} random
response vector $\y$, 
  \begin{align*}
    \E_{S,\wbh,\y}\big[\wbh(\y_S)\big] = \w^*\quad\text{and}\quad
    \mathrm{MSPE}\big[\wbh(\y_S)\big]-\mathrm{MSPE}\big[\wols(\y|\X)\big]\leq
    \epsilon\cdot\E_\y\big[\|\xib_{\y|\X}\|^2\big].
  \end{align*}
\end{theorem}
In the statistical setting, minimizing the MSPE is often referred to
as V-optimal design \citep[see][]{v-optimal}.
On the other hand, in worst-case regression analysis (when responses
form a fixed vector $\y\in\R^n$), the mean squared prediction error is
often replaced by the ``square loss'': $L(\w) =
\|\X\w-\y\|^2$. Theorem~\ref{t:mspe} implies a bound on the expected
square loss of the estimator $\wbh(\y_S)$:
\begin{align}
    \E_{S,\wbh}\big[L(\wbh(\y_S))\big] \overset{(*)}{=} \MSPE{\wbh(\y_S)}+L(\w^*) \leq
(1+\epsilon)\cdot L(\w^*).
    \label{e:regret}
\end{align}
where $(*)$ follows from the unbiasedness of $\wbh(\y_S)$ via the
bias-variance decomposition of the expected square loss.
The only
\textit{expected} loss bound of this kind known prior to this result
required sample size $k=O(d^2/\epsilon)$ \citep{unbiased-estimates}. 

Since our experimental design is randomized, each evaluation may
produce a different result. In fact this can go to our advantage:
instead of using one design with a larger $k$ we can choose to produce
multiple independent designs with a small $k$, say
$(S_1,\wbh_1),\dots,(S_m,\wbh_m)$, and then average them. This
strategy may be preferrable in distributed settings and when data
privacy is a concern. Since all the
designs are unbiased for the random responses $\y$, it follows that:
\begin{align*}
  \mathrm{MSE}\bigg[\frac1m\sum_{t=1}^m\wbh_t(\y_{S_t})\bigg] -
  \mathrm{MSE}\big[\wols(\y|\X)\big] = 
\frac1m\Big(\mathrm{MSE}\big[\wbh_1(\y_{S_1})\big] - \mathrm{MSE}\big[\wols(\y|\X)\big]\Big),
\end{align*}
with an analogous formula also holding for the MSPE.

\subsection{Minimax-optimal experimental design}
\label{s:minimax}
If we divide both sides of the inequality in Theorem~\ref{t:mse} by
the right-hand-side $\E_\y[\|\xib_{\y|\X}\|^2]$, we see a ratio bounded above by $\epsilon$
for all $\y$. This ratio, or rather its maximum over all $\y$, can be considered a
quality criterion for experimental designs,  to be minimized instead of only
bounded. We will call the optimum a
\emph{minimax-optimal design}. The key difference compared to the standard setup
is that we allow the design to be randomized. Let $\Fc$ denote the family
of \textit{all} random vectors in $\R^n$ with finite second moment.
\begin{definition}
Given matrix $\X\in\R^{n\times d}$ and budget $k\in\{d,\dots,n\}$, let
$\Wc_k(\X)$ be the family of all random experimental designs $(S,\wbh)$ of
size $k$ such that: 
\begin{align*}
  \E_{S,\wbh,\y}\big[\wbh(\y_S)\big]
    = \argmin_\w\E\big[\|\X\w-\y\|^2\big]
    = \X^\dagger\E[\y]\quad\text{for
  all }\ \y\in\Fc.
\end{align*}
\end{definition}
In Appendix \ref{a:minimax} we show that the least squares estimator
$\wols(\y|\X)=\X^\dagger\y$ is 
the \textit{minimum variance unbiased estimator} (MVUE) among all
estimators with unrestricted budget, i.e., $\Wc_n(\X)$. 
\begin{proposition}\label{p:mvue}
Given any full rank matrix $\X\in\R^{n\times d}$ and any random function
$\wbh:\R^n\rightarrow\R^d$,
\begin{align*}
\text{if}\quad
  \E_{\y,\wbh}[\wbh(\y)]=\X^\dagger\E[\y]\quad \forall_{\y\in\Fc},
  \quad\text{then}\quad
  \Var\big[\wbh(\y)\big]\succeq\Var\big[\wols(\y|\X)\big]\quad \forall_{\y\in\Fc}.
\end{align*}
\end{proposition}
It is thus natural to
minimize the excess mean squared error incurred by an
unbiased estimator with a restricted budget compared to that of
$\wols(\y|\X)$.
Since we take a maximum over all response vectors in $\Fc$, we
normalize this by the noise
$\E_\y\big[\|\xib_{\y|\X}\|^2\big]$ (equal to $n\sigma^2$ in the
classical setting).
To avoid division by zero, we exclude all fixed vectors in the
  column span of $\X$, denoted $\mathrm{Sp}(\X)\subseteq\R^n$.
  \vspace{-2mm}
\begin{definition}\label{d:minimax}
Let the minimax-optimal value of
experimental design for $\X\in\R^{n\times d}$,  $d\leq k\leq n$
be:  
\begin{align*}
R_k^*(\X)&\defeq
  \min_{(S,\wbh)\in\Wc_k(\X)}\ \max_{\y\in\Fc\backslash\mathrm{Sp}(\X)}\,\frac{\mathrm{MSE}\big[\wbh(\y_S)\big]-
     \mathrm{MSE}\big[\wols(\y|\X)\big]
  }{\E_\y\big[\|\xib_{\y|\X}\|^2\big]},
\end{align*}
where $\MSE{\wbh}$ for any unbiased estimator $\wbh$ denotes $\E\big[\|\wbh-\E[\wbh]\|^2\big]$.
\end{definition}
\begin{proposition}\label{p:minimax-bounded}
The following are true if $\X$ denotes a full rank $n\times d$ matrix and
$\phi=\tr\big((\X^\top\X)^{-1}\big)$:

1. \  For any $\X$ and $d\leq k\leq n$,  we have $0\leq
R_k^*(\X)<\infty$;
    
2. \  There is $C>0$ such that for any $\X$ and
$k\geq C\cdot d\log n$, we have $R_k^*(\X)\leq C\cdot \phi/k$;

3. \ For any $n$, $d$ and $\epsilon\in(0,1)$, there is $\X$
    s.t.~if $k^2<\epsilon nd/3$ then $R_k^*(\X)\geq
    (1-\epsilon)\cdot\phi/k$.
\end{proposition}
	\vspace{-1mm}
Part 2 of Proposition \ref{p:minimax-bounded} is an immediate
corollary of Theorem \ref{t:mse}, whereas part 3 is an application of a
matrix inequality of \cite{avron-boutsidis13}, see details in Appendix \ref{a:minimax}. Note that if we defined $\Fc$ as the family of all
random vectors $\y$ such that the noise  $\xib_{\y|\X}$ is
i.i.d.~centered Gaussian (with any variance), then in this case the
least squares estimator would also be the MVUE, and the above
definition would become equivalent to the classical A-optimality criterion.
Even in this special case, finding an exactly optimal design is hard
(to our knowledge, NP-hardness has not been established),
although efficient approximation algorithms exist for A-optimality (see Section
\ref{s:related-work}). Similar questions can be asked about
minimax-optimal design, however without any
restrictions on the design procedure, this task appears daunting. In
Section \ref{s:vol} we present one such restriction based on
``volume sampling'' which leads to a family of efficient
unbiased estimators that we used in Theorems \ref{t:mse} and
\ref{t:mspe}.


\subsection{Construction and efficiency of random experimental designs}\label{ss:RED}
The random experimental design used in Theorems
\ref{t:mse} and \ref{t:mspe} consists of two primary components:
	\vspace{-1mm}
\begin{enumerate}
  \item \textit{volume sampling:} the initial few experiments are drawn
    from a joint sampling distribution over 
    sets $S\subseteq [n]$ of size $d$ such that $\Pr(S)\propto
    \det(\X_S)^2$;
	\vspace{-1mm}
  \item \textit{i.i.d.~sampling:} the remaining $k-d$ experiments are
    sampled independently from a carefully chosen distribution
    $q=(q_1,\dots,q_n)$. 
  \end{enumerate}
	\vspace{-1mm}
  While it is mainly the i.i.d.~sampling that is responsible for 
bounding the sample size $k$, volume sampling is
  necessary for establishing both the unbiasedness and the expected
  bounds. The key novelty of our analysis is using volume sampling to
  control the MSE in the \textit{tail} of the
  distribution, and using the concentration properties of
  i.i.d.~sampling to bound it in the \textit{bulk} of the
  distribution (see Section \ref{s:proofs}). The i.i.d.~sampling
  distribution $q$ used in the proof is a mixture of uniform
  distribution with two importance sampling techniques:
  \begin{enumerate}
  \item  \textit{Leverage score sampling}: $\Pr(i)=p_i^{\mathrm{lev}}
\defeq \frac1d \x_i^\top(\X^\top\X)^{-1}\x_i$ for
  $i\in[n]$. This is a standard sampling method which has been used in
  obtaining bounds for worst-case linear regression.
  \item \textit{Inverse score sampling}: $\Pr(i)= p_i^{\mathrm{inv}}\defeq
    \frac1\phi\x_i^\top(\X^\top\X)^{-2}\x_i$ for
    $i\in[n]$. This is a novel sampling technique which is essential
    for achieving $O(\phi/\epsilon)$ sample size for small $\epsilon$.
  \end{enumerate}
As discussed earlier, having chosen a design, we may wish to
produce multiple independent samples of it, for example to construct
an averaged estimator. Thus, we break down the computational cost into
the preprocessing cost (incurred once per given matrix $\X$) and
sampling/estimation cost (incurred every time a new estimator is produced).
The estimation step simply requires computing a least squares estimator
from $k$ samples, which costs $O(kd^2)$. The
preprocessing involves all the calculations necessary to construct the
sampling distributions. Both of the above importance
sampling distributions can be computed exactly in time $O(nd^2)$ or
approximately in time $O(nd\log n +d^3\log d)$ using standard sketching
techniques \citep[see][]{DMMW12_JMLR}. Once they are obtained, the sampling cost is
negligible. On the other hand, for volume sampling both preprocessing
and sampling cost can be significant.
\cite{leveraged-volume-sampling} showed that a volume sampled set of size
$d$ can be generated in time $O(d^4)$ by selecting it from a sequence
of $O(d^2)$ i.i.d.~samples from the leverage score distribution
$p^{\mathrm{lev}}$ (see Theorem 6 there). We improve on
this in the following result.
\begin{theorem}\label{t:bottom-up}
  For any $\X$ and $q$ such that $q_i\geq
  \frac12 p_i^{\mathrm{lev}}$, there is an algorithm which,
  given matrix $\X^\top\X$ and a stream of i.i.d~samples from
  $q$, returns a set $S$
  s.t.~$\Pr(S)\propto \det(\X_S)^2$, and w.p.~at least $1-\delta$
  it runs in time $O\big(d^3\log d \log\frac1\delta\big)$ using
  $O(d\log d\log\frac1\delta)$ i.i.d.~samples. 
\end{theorem}
Our algorithm improves on the best known sampling time for volume
sampling from $O(d^4)$ to $O(d^3\log d)$, which has important
implications for other applications of this distribution such as
determinantal point processes (see Section \ref{s:vol} for the proof
and further discussion). To establish correctness of the sampling,
this algorithm requires the exact computation of matrix $\X^\top\X$,
which typically costs $O(nd^2)$. The algorithm of
\cite{leveraged-volume-sampling} only requires an approximation of
this matrix, which can be computed in time $O(nd\log n + d^4\log d)$.
Similar improvements in the preprocessing cost for our volume sampling
algorithm may be possible, but we leave this as an open question.

\section{Related work}
\label{s:related-work}

There is a large body of related work, and we describe only that which
most informed our approach.

\paragraph{Classical experimental design.}

Many optimality criteria have been considered as functions $F(S)$ of a
subset $S\subseteq [n]$ \citep{optimal-design-pukelsheim}, assuming
that $y_i=\x_i^\top\w^*+\xi_i$ where $\xi_i\sim\Nc(0,\sigma^2)$, and
these typically have natural interpretations for the least squares
estimator.  
Recent work has studied the tractability of finding an approximately
optimal subset $\widehat{S}$ of size $k$, i.e., such that
$F(\widehat{S}) \leq (1+\epsilon) \min_{S:\,|S|=k}F(S)$.  
For example, 
\cite{near-optimal-design} showed that polynomial time algorithms are
possible for many classical optimality criteria, such as A-optimality,
$F_A(S)=\tr((\X_S^\top\X_S)^{-1})$, D-optimality,
$F_D(S)=\det(\X_S^\top\X_S)^{-1}$, V-optimality,
$F_V(S)=\tr(\X(\X_S^\top\X_S)^{-1}\X^\top)$ and others, as long as
$k=\Omega(d/\epsilon^2)$;  
\cite{tractable-experimental-design} showed tractable approximability of A/V-optimality for $k=\Omega(d^2/\epsilon)$; and 
this was later improved by \cite{proportional-volume-sampling} to $k=\Omega(d/\epsilon + (\log\epsilon^{-1})/\epsilon^2)$. 
Robust variants of experimental design have been considered to address more general response models. 
In particular, 
\cite{robust-design} assume that the covariance matrix of the noise is
known only approximately and defines a minimax-type criterion where
the maximization goes over a neighborhood of that covariance; and  
\cite{v-optimal} use an active learning procedure to estimate the 
individual noise variances before constructing the design.  
None of these procedures, however, are truly agnostic to the response
model.

\paragraph{Subset selection for worst-case regression.} 

Subset selection has been studied extensively for both statistical and worst-case regression models.
Perhaps most relevant is the work of \cite{coresets-regression}, which
showed a lower bound for any deterministically chosen subset $S$ and
function $\wbh$, when the hidden response vector $\y$ is arbitrary but
fixed.
This implies that random sampling is necessary in this setting.
In the context of \emph{randomized numerical linear algebra}
\citep[RandNLA; see][]{woodruff2014sketching, DM16_CACM}, it was shown by
\cite{drineas2006sampling} that a random sampling algorithm based on the statistical leverage scores
constructs an estimator $\wbh(\y_S)$ which, with constant probability, achieves $\|\wbh(\y_S)-\w^*\|^2\leq\epsilon\cdot \|\X\w^*-\y\|^2$ by using $k=O(d\log d + \lambda_{\max}((\X^\top\X)^{-1})\cdot d/\epsilon)$ samples.
The estimator we propose in Theorem \ref{t:mse} achieves the same
bound for $k=O(d\log d + \phi/\epsilon)$. 
Since 
$ 
\phi=\tr\big((\X^\top\X)^{-1}\big) \leq
  \lambda_{\max}\big((\X^\top\X)^{-1}\big)\cdot d \leq   \phi\cdot d  ,
$ 
our result is better by up to a factor of $d$.


\paragraph{Statistical versus algorithmic approaches.}

 \cite{ping-ma2014} and \cite{garvesh2015} were the first to 
 consider statistical guarantees (such as A-optimality) that can be
 obtained by sampling methods developed for RandNLA (primarily leverage
 score sampling), contrasting them with some common worst-case guarantees.  
However, these works treat those two settings separately (in particular, the statistical setting is limited to i.i.d.~Gaussian noise), rather than putting them under one umbrella of minimax experimental design, as we do.
Subsequently, \cite{chen2017condition} showed loss bounds for
worst-case regression which extend to a randomized response model
that is comparable to ours. 
They give a randomized estimator 
({\em not unbiased}) that with constant probability
achieves the following bound on the square loss:
$L(\wbh(\y_S))\leq(1+\epsilon)L(\w^*)$, for sample size
$k=O(d/\epsilon)$, where $L(\w)=\|\X\w-\y\|^2$. 
In contrast we obtain an {\em unbiased} estimator achieving
the same bound \emph{in expectation} with only slightly larger
sample size of $k=O(d\log n+d/\epsilon)$.

\paragraph{Constant probability versus unbiased expectations.}
Unlike our Theorems \ref{t:mse} and \ref{t:mspe}, most results in
RandNLA are stated to hold with high or constant probability
\citep{woodruff2014sketching, DM16_CACM, RandNLA_PCMIchapter_TR} as opposed to in 
expectation, and they do not provide unbiased estimators, which often
makes them incomparable to statistical approaches.
In fact, expected bounds are often impossible for these
techniques \citep[e.g., for leverage score sampling; see][]{unbiased-estimates-journal}.
Unbiased estimators were first introduced to worst-case regression by
\cite{unbiased-estimates}, who gave the first \textit{expected}
square loss bound for sample size of $k=O(d^2/\epsilon)$ via volume
sampling.  
Subsequently, \cite{leveraged-volume-sampling} demonstrated an
unbiased estimator with a \textit{constant probability} loss bound for
sample size $k=O(d\log d+d/\epsilon)$. Our result builds on the latter
by obtaining an unbiased estimator with an \textit{expected} loss
bound for $k=O(d\log n+d/\epsilon)$.

\section{Rescaled volume sampling}\label{s:vol}

We now discuss the sampling distribution introduced by
\cite{leveraged-volume-sampling}, based on earlier work by
\cite{avron-boutsidis13}, that allows for constructing
unbiased least squares estimators. 
\begin{definition}
  \label{d:vs}
  Given full rank matrix $\X\in\R^{n\times d}$ and a distribution
  $q=(q_1,\dots,q_n)$ s.t.~$q_i> 0$ for all $i\in[n]$, we define
$q$-rescaled volume sampling of size $k\geq d$, written
$\Vol_q^k(\X)$, as a distribution 
  over index sequences $\pi=(\pi_1,\dots,\pi_k)\in[n]^k$ such that:
  \begin{align*}
    \Pr(\pi) =
    \frac{\det\!\big(\X^\top\S_\pi^\top\S_\pi\X\big)}{\frac{d!}{k^d}{k\choose
    d}\det(\X^\top\X)}\prod_{i=1}^kq_{\pi_i},\quad\text{where}\quad
    \S_\pi = \begin{bmatrix}\frac1{\sqrt{k q_{\pi_1}}}\e_{\pi_1}^\top\\
      \vdots\\\frac1{\sqrt{k q_{\pi_k}}}\e_{\pi_k}^\top\end{bmatrix}
    \in\R^{k\times n}. 
  \end{align*}
\end{definition}
It is easy to see that for $k=d$, the matrix $\S_\pi\X$ is square 
and thus
\begin{align*}
  \det(\X^\top\S_\pi^\top\S_\pi\X)
    =\det(\S_\pi\X)^2=\frac{\det(\X_\pi)^2}{d^d\prod_i q_{\pi_i}},
\end{align*}
where $\X_\pi$ selects the rows indexed by $\pi$ from $\X$. So the
distribution $\Vol_q^d(\X)$ is the same for every $q$. For this
reason, we will write it simply as $\Vol^d(\X)$.
We mention the following
recently shown results regarding rescaled volume sampling which we use
later in the proofs. 
\begin{lemma}[\citeauthor{leveraged-volume-sampling}, \citeyear{leveraged-volume-sampling}]\label{l:expectations}
For any $\X$, $q$ and $k$ as in Definition \ref{d:vs}, if
$\pi\sim\Vol_q^k(\X)$, then  
\begin{align}
  \E\big[(\S_\pi\X)^\dagger \S_\pi\big] &= \X^\dagger,\label{eq:unbiased}\\
  \E\big[(\X^\top\S_\pi^\top\S_\pi\X)^{-1}\big]
  &\preceq \frac k{k\!-\!d\!+\!1}(\X^\top\X)^{-1}.\label{eq:square-inverse}
\end{align}
\end{lemma}
The following random experimental design emerges as a natural
candidate for proving Theorem \ref{t:mse}:
\begin{align}
  S = \{\pi_1\}\cup\dots\cup\{\pi_k\},\quad 
    \wbh(\y_S) = (\S_\pi\X)^ \dagger\S_\pi\y
    ,\quad\text{where }\pi\sim\Vol_q^k(\X).\label{eq:design}
\end{align}
Note that since sequence $\pi$ may include repetitions whereas set $S$
may not (it is not a multi-set), the function 
$\wbh(\y_S)$ has to depend not only on $\y_S$ but also on the multiplicities of each
response in sequence $\pi$. Thus both set $S$ and function $\wbh(\cdot)$ 
in this design are in fact randomized and satisfy Definition \ref{d:rand}. 
Lemma \ref{l:expectations} shows that this design 
is unbiased for any $\y$, leading to a restricted notion of 
minimax-optimality which provides an upper-bound on $R_k^*(\X)$ (proof in
Appendix \ref{a:minimax}):
\begin{lemma}\label{l:reduction}
  Let $\Vc_k(\X)$ consist of all random experimental designs
  based on $q$-rescaled volume sampling as in \eqref{eq:design},
  parameterized by  distribution $q$, and let $\mathrm{Sp}(\X)$ be the
  column span of $\X$. Then: 
\begin{align*}
R_k^*(\X)\leq \min_{(S,\wbh)\in\Vc_k(\X)}\ 
  \max_{\y\in\R^n\backslash \mathrm{Sp}(\X)}\frac{\E_{S,\wbh}\big[\|\wbh(\y_S)-\wols(\y|\X)\|^2\big]}{\|\X\,\wols(\y|\X)-\y\|^2}.
\end{align*}
\end{lemma}
Even for the restricted minimax-optimality, finding the exact or even
approximate optimum $q^*$ is open. However,
in Section \ref{s:proofs} we bound the restricted minimax value by
selecting a particular distribution $q$ and utilizing the following
decomposition of $q-$rescaled volume sampling in the analysis. 
\begin{lemma}[\citeauthor{correcting-bias}, \citeyear{correcting-bias}]\label{l:augmenting}
For any $\X$, $q$ and $k$ as in Definition \ref{d:vs}, let
$\pi\sim\Vol^d(\X)$ and $\tilde{\pi}_1,\dots,\tilde{\pi}_{k-d}\simiid
q$. Finally let $\sigma$  
be a permutation of $(1,\dots,k)$ drawn uniformly at random. Then:
\begin{align*}
  \sigma\big(\pi_1,\dots,\pi_d,\tilde{\pi}_1,\dots,\tilde{\pi}_{k-d}\big)\sim
  \Vol_q^k(\X). 
\end{align*}
\end{lemma}

\begin{wrapfigure}{r}{0.45\textwidth}
  \vspace{-3.5mm}
  \centering
\begin{minipage}{0.45\textwidth}
\floatname{algorithm}{\small Algorithm}
\begin{algorithm}[H] 
    \caption{\small (Bottom-up) volume sampling}\label{alg:bottom-up}
  \begin{algorithmic}[0]
    \STATE \textbf{input:} $\X\in\R^{n\times d}$, \ $q$, \ 
    $\A_1=(\X^\top\X)^{-1}$ 
    \STATE \textbf{output:} $\pi\sim\Vol^d(\X)$
    \STATE \textbf{for }$i=1..d$
    \STATE \quad\textbf{repeat}
    \STATE  \quad\quad Sample $\pi_i \sim q$
    \STATE \quad\quad Sample $a\sim
    \text{Bernoulli}\Big(\frac{\x_{\pi_i}^\top\A_i\x_{\pi_i}}{2d \,q_{\pi_i}}\Big)$
    \STATE \quad\textbf{until} $a=1$
    \STATE \quad$\A_{i+1} \leftarrow \A_i -
    \frac{\A_i\x_{\pi_i}\x_{\pi_i}^\top\A_i}{\x_{\pi_i}^\top\A_i\x_{\pi_i}}$
    \STATE \textbf{end for}
    \RETURN $\pi_1,\dots,\pi_d$
  \end{algorithmic}
\end{algorithm}
\end{minipage}
\vspace{-4mm}
\end{wrapfigure}
If distribution $q$
is sufficiently close to the leverage score sampling distribution
$p^{\mathrm{lev}}$, then even the initial volume sample of size $d$ can
be \textit{selected out of} an i.i.d.~sample of size $O(d\log d)$ as
shown in Algorithm \ref{alg:bottom-up}. This algorithm is a new
implementation of a classical method for sampling from a so-called
elementary determinantal point process, due to
\cite{dpp-independence}. To our knowledge, the best previously known
runtime for this method was $O(nd^2)$ for each produced volume sample
(see \cite{dpp-coreset}), whereas the runtime of this implementation is
$O(d^3\log d)$ (in addition to a preprocessing step which involves
computing distribution $q$ and matrix $(\X^\top\X)^{-1}$). Note that
for some applications of volume sampling, such as determinantal point
process sampling, one can often assume that $\X^\top\X=\I$ 
\citep[see][]{dpp-intermediate}, in which case the preprocessing becomes much
cheaper than $O(nd^2)$. We now prove Theorem
\ref{t:bottom-up} by establishing correctness and runtime of Algorithm
\ref{alg:bottom-up}.

\begin{proofof}{Theorem}{\ref{t:bottom-up}}
Since $\A_i\preceq (\X^\top\X)^{-1}$, then $\x_{\pi_i}^\top\A_i\x_{\pi_i}/(2d\, q_{\pi_i})
\leq p_{\pi_i}^{\mathrm{lev}}/(2q_{\pi_i})\leq 1$ is a valid
  Bernoulli probability. We start with the proof of correctness, which
  is an adaptation of the one given by \cite{dpp-independence}.
For any $i,j$ define
$\u_j^{(i)}=\A_{i}^{\sfrac12}\x_{j}$. The
marginal probability of sampling $\pi_{i+1}$ conditioned on previous steps
is proportional to $\|\u_{\pi_{i+1}}^{(i+1)}\|^2$, which can be written as:
\begin{align*}
  \|\u_{\pi_{i+1}}^{(i+1)}\|^2&=  \x_{\pi_{i+1}}^\top\A_{i+1}\x_{\pi_{i+1}}
  =
  \x_{\pi_{i+1}}^\top\A_{i}^{\sfrac12}\bigg(\I -
    \frac{\A_{i}^{\sfrac12}\x_{\pi_i}\x_{\pi_i}^\top\A_{i}^{\sfrac12}}
    {\x_{\pi_i}^\top\A_{i}\x_{\pi_i}}\bigg)\A_{i}^{\sfrac12}\x_{\pi_{i+1}}\\
  &=\u_{\pi_{i+1}}^{(i)\top}\bigg(\I-\frac{\u_{\pi_i}^{(i)}\u_{\pi_i}^{(i)\top}}
    {\|\u_{\pi_i}^{(i)}\|^2}\bigg)\u_{\pi_{i+1}}^{(i)}
    =\big\|\P_i\u_{\pi_{i+1}}^{(i)}\big\|^2,
    \quad\text{where }\ \P_i=\I-\frac{\u_{\pi_i}^{(i)}\u_{\pi_i}^{(i)\top}}
    {\|\u_{\pi_i}^{(i)}\|^2}
\end{align*}
is a projection onto the $(d-1)$-dimensional
subspace of $\R^d$ orthogonal to $\u_{\pi_i}^{(i)}$. We conclude that
vectors $\u_{j}^{(i)}$ are obtained from $\u_{j}^{(1)}=(\X^\top\X)^{-\sfrac12}\x_j$ by
repeatedly \textit{projecting away} the points that were already
sampled. This means that since $\U=\X (\X^\top\X)^{-\sfrac12}$
satisfies $\U^\top\U=\I$, we have: 
\begin{align*}
  \sum_{j=1}^n\|\u_{j}^{(i+1)}\|^2
  =\tr\bigg(\sum_{j=1}^n\u_j^{(1)}\u_j^{(1)\top}\cdot\prod_{t=1}^{i}\P_t\bigg)
  = \tr\bigg(\U^\top\U\cdot\prod_{t=1}^{i}\P_t\bigg)= d-i.
\end{align*}
We can now write the probability of sampling a sequence
$\pi_1,\dots,\pi_d$ as:
\begin{align*}
  \Pr(\pi) = \prod_{i=1}^d\frac{\|\u_{\pi_i}^{(i)}\|^2}{d-i+1} =
  \frac{\det(\U_\pi)^2}{d!} = \frac{\det(\X_\pi)^2}{d!\det(\X^\top\X)},
\end{align*}
which follows because $\det(\U_\pi)^2$ is the squared volume spanned
by the vectors $\u^{(1)}_{\pi_1},\dots,\u^{(1)}_{\pi_d}$ and it is obtained as a
series of applications of the ``base $\times$ height'' formula. To
bound the runtime, we note that the expected acceptance probability in
the $i$th step of Algorithm \ref{alg:bottom-up} is:
\begin{align*}
  \sum_{j=1}^nq_j\cdot \frac{\x_{j}^\top\A_i\x_j}{2d
  \,q_j}=\frac1{2d}\sum_{j=1}^n\|\u_{j}^{(i)}\|^2 = \frac{d-i+1}{2d}.
\end{align*}
Thus, the expected total number of trials of rejection sampling
throughout the algorithm is:
\begin{align*}
  \sum_{i=1}^d\frac{2d}{d-i+1} = 2d\sum_{i=1}^d\frac1i\leq 2d\big(\ln(d)+1\big).
\end{align*}
Standard tail bounds for a sum of geometric random variables show that
with probability at least $1-\delta$ the number of rejection sampling
trials is $O(d\log d\log\frac1\delta)$. Each trial costs
$O(d^2)$, as does updating the matrix $\A_i$, which concludes the proof.
\end{proofof}

\section{Proof of Theorem \ref{t:mse}}
\label{s:proofs}
In this section we use $\wbh$ and $\wols$ as shorthands for
$\wbh(\y_S)$ and $\wols(\y|\X)$. 
To prove the error bound in Theorem \ref{t:mse} we will invoke Lemma
\ref{l:reduction},
thereby restricting ourselves to a fixed response vector $\y\in\R^n$,
in which case $\w^*=\wols$, and a volume sampled random design
as discussed in the previous section.
The construction in our proof uses leverage scores and inverse scores, as
discussed in Subsection \ref{ss:RED}.
\begin{definition}
Given full rank matrix $\X$, its $i$th \textbf{ leverage score} is
defined as $l_i(\X)\defeq \x_i^\top(\X^\top\X)^{-1}\x_i$,
and its $i$th \textbf{inverse score} as  $v_i(\X) \defeq \x_i^\top(\X^\top\X)^{-2}\x_i$.
\end{definition}
The key challenge in obtaining the result
is that standard techniques developed for
i.i.d.~sampling \citep[see, e.g., ][]{drineas2006sampling} only show the
least squares error bounds with constant probability. 
Such bounds do not suffice to show an expected bound because
  we do not have  control over what happens in the \textit{failure 
    event} (where the expectation may be unbounded). In fact, an
  expected bound of this type is not possible 
  for any i.i.d.~sampling 
  \citep[see Proposition 11 in][]{unbiased-estimates-journal}. Our key contribution is to
  define an event $A$ s.t.:
  \begin{enumerate}
    \item if $A$ occurs, then we can show a strong expected
      bound relying on i.i.d.~sampling techniques,
    \item if $A$ fails to occur, a weaker bound still holds
      because of the jointness of volume sampling.
    \end{enumerate}
    Crucially, the probability of failure will be exponentially small,
    thus allowing us to obtain the desired result. This technique is
    described in the proof  of the following key lemma.
  \begin{lemma}\label{l:key}
  There is $C>0$ s.t.~for any full rank
matrix  $\X\in\R^{n\times d}$, if
  $\pi\sim\Vol_{q(\alpha)}^k(\X)$ where
  \begin{align*}
 q(\alpha) &=
  \alpha \,\big(0.5\cdot p^{\mathrm{uni}}+0.5\cdot p^{\mathrm{inv}}\big) 
  \,+\,(1-\alpha)\,p^{\mathrm{lev}}  
  \quad\text{for}\quad \alpha\in\big[0.5,0.75\big],\\
  \text{with}\quad p_i^{\mathrm{uni}}
     &= 1/n,\quad
 p_i^{\mathrm{inv}}= v_i(\X)/\phi,
\quad	p_i^{\mathrm{lev}} = l_i(\X)/d,\ \text{ and }\ \phi=\tr\big((\X^\top\X)^{-1}\big),
  \end{align*}
then for any $k\geq d+ C\max\{ d\log n , \phi/\epsilon\big\}$ and an arbitrary
vector $\xib\in\R^n$ we have
\begin{align*}
 \E\big[\|(\S_\pi\X)^\dagger\S_\pi\xib\|^2\big] \leq \frac\epsilon 8
  \|\xib\|^2 +4\, \big\|\X^\dagger\E[\S_\pi^\top\S_\pi]\xib\big\|^2.
\end{align*}
\end{lemma}
\begin{proof}
Observe that we chose $q(\alpha)$ as a
mixture of three distributions in such a way that each of them
has at least 0.25 weight in the mixture. Lemma
    \ref{l:augmenting} allows us to decompose sample $\pi$ into
    the volume part, $\pi_{[d]}=(\pi_1,\dots,\pi_d)\sim \Vol^d(\X)$,
    and the i.i.d.~part, 
    $\tilde{\pi}=(\pi_{d+1},\dots,\pi_k)\sim q $ (technically, this
    requires reordering the sequence $\pi$). We now define an event $A$ as a variant of the so-called
    \textit{subspace  embedding} condition:
    \begin{align*}
      \text{event $A$ holds iff }\quad
\frac1k\sum_{i=d+1}^k\frac1{q_{\pi_i}}\x_{\pi_i}\x_{\pi_i}^\top\
      \succeq\ \frac12\,\X^\top\X.  
    \end{align*}
    Note that $A$ is  defined only over the i.i.d.~samples
    $\tilde{\pi}$ and is therefore completely independent of the volume sample
    $\pi_{[d]}$.
We start by decomposing the expectation
into two terms: 
    \begin{align}
      \E\big[\|(\S_\pi\X)^\dagger\S_\pi\xib\|^2\big] =
      \Pr(A)\,\E\big[\|(\S_\pi\X)^\dagger\S_\pi\xib\|^2\ |\,A\big] + \Pr(\neg
      A)\,\E\big[\|(\S_\pi\X)^\dagger\S_\pi\xib\|^2\ |\,\neg A\big].\label{eq:terms}
    \end{align}
To bound the first term we decompose the squared norm into
two factors 
    \begin{align*}
\|(\S_\pi\X)^\dagger\S_\pi\xib\|^2
      &= \big\|(\X^\top\S_\pi^\top\S_\pi\X)^{-1}
        \X^\top\S_\pi^\top\S_\pi\xib\big\|^2
    \\ & \leq \big\|(\X^\top\S_\pi^\top\S_\pi\X)^{-1}\X^\top\X\big\|^2
        \cdot\big\|(\X^\top\X)^{-1}\X^\top\S_\pi^\top\S_\pi\xib\big\|^2.
    \end{align*}
   When event $A$ occurs, then the first factor can be easily bounded by 4, because
   \begin{align*}
(\X^\top\S_\pi^\top\S_\pi\X)^{-1}\X^\top\X\preceq
     \bigg(\frac1k\sum_{i=d+1}^k\frac1{q_{\pi_i}}\x_{\pi_i}\x_{\pi_i}^\top\bigg)^{\!-1}\X^\top\X\preceq 2\,\I.
   \end{align*}
Thus, it remains to bound the second factor in expectation,
i.e.
$\E\big[\|\X^\dagger\S_\pi^\top\S_\pi\xib\|^2\,|A\big]$. For this,
we 
   need an extension of a result by \cite{leveraged-volume-sampling}
(see proof in Appendix \ref{a:mult}).
   \begin{lemma}\label{l:mult}
  Given $\A\in\R^{n\times d}$,
  $\b\in\R^n$ and $\beta>0$, let $\pi\sim\Vol_q^k(\A)$ where $q_i\geq 
  \beta\max\!\big\{\frac{\|\a_i\|^2}{\|\A\|_F^2},\frac{l_i(\A)}{d}\big\}$
  for all $i\in [n]$ and $k\geq d$. Then
  $\v_\pi=\A^\top\S_\pi^\top\S_\pi\b$ satisfies 
  \begin{align*}
\E\big[\|\v_\pi-\E[\v_\pi]\|^2\big]\leq
    \frac1{\beta^2 k}\|\A\|_F^2\|\b\|^2.
  \end{align*}
\end{lemma}
We use Lemma \ref{l:mult} with 
$\A=\X^{\dagger\top}$ and $\b=\xib$. In this case
$\|\a_i\|^2=v_i(\X)$ and $l_i(\A)=l_i(\X)$. Also let
$\beta=0.25$ and observe that $\phi=\|\X^\dagger\|_F^2$.
Now for $k\ge C\cdot\phi/\epsilon$ and $C\geq 16\cdot4/\beta^2$ we have:
 \begin{align}
   \Pr(A) \,\E\big[\|(\S_\pi\X)^\dagger\S_\pi\xib\|^2\ |\,A\big]
   &\leq 4\ \Pr(A)\ \E\Big[\big\| \X^\dagger\S_\pi^\top\S_\pi\xib\big\|^2\ |\,A\Big] 
\nonumber \\
   &\leq 4\, \E\big[\|\v_\pi\|^2\big]\ = \ 4\,\E\Big[\big\|\v_\pi-\E[\v_\pi]\big\|^2\Big] + 4\,\big\|\E[\v_\pi]\big\|^2\nonumber  \\
   &\leq \frac{4\epsilon}{C\beta^2\phi}\|\X^\dagger\|_F^2\|\xib\|^2 +
     4\,\big\|
     \X^\dagger\E[\S_\pi^\top\S_\pi]\xib\big\|^2\nonumber \\
     &\leq \frac{\epsilon}{16}\|\xib\|^2+4\,\big\|
     \X^\dagger\E[\S_\pi^\top\S_\pi]\xib\big\|^2. \label{eq term 1}
 \end{align}
 Next, we bound the second term in \eqref{eq:terms} by using a
 different decomposition of $\|(\S_\pi\X)^\dagger\S_\pi\xib\|^2$. We will
 use the fact that since $q_i(\alpha)\geq \frac14
 p_i^{\mathrm{uni}}=\frac1{4n}$ for all $i\in[n]$, then
 $\S_\pi^\top\S_\pi\preceq 4n\,\I$ for all $\pi\in[n]^k$. It is only
 here that we use the $p^{\mathrm{uni}}$ term in $q(\alpha)$. It
 follows that
 \begin{align*}
\|(\S_\pi\X)^\dagger\S_\pi\xib\|^2
   &= \big\|(\X^\top\S_\pi^\top\S_\pi\X)^{-1}
     \X^\top\S_\pi^\top\S_\pi\xib\big\|^2\\
   &\leq
     \|\xib\|^2\|\S_\pi^\top\S_\pi\X(\X^\top\S_\pi^\top\S_\pi\X)^{-2}\X^\top\S_\pi^\top\S_\pi\|\\
   &\leq
     \|\xib\|^2\tr\big(\S_\pi^\top\S_\pi\X(\X^\top\S_\pi^\top\S_\pi\X)^{-2}\X^\top\S_\pi^\top\S_\pi\big)\\
   &=
\|\xib\|^2\tr\big(\X^\top(\S_\pi^\top\S_\pi)^2\X(\X^\top\S_\pi^\top\S_\pi\X)^{-2}\big)\\ 
   &\leq \|\xib\|^24n\,\tr\big((\X^\top\S_\pi^\top\S_\pi\X)^{-1}\big)
     \leq 4n \|\xib\|^2\,\tr\big((\X^\top\S_{\pi}^\top\I_{[d]}\S_{\pi}\X)^{-1}\big),
 \end{align*}
 where $\I_{[d]}=\sum_{i=1}^d\e_i\e_i^\top$ selects the
 first $d$ rows from $\S_\pi$.
 Since $\pi_{[d]}\sim \Vol^d(\X)$ and it is independent of the event
 $A$, from inequality \eqref{eq:square-inverse} in Lemma \ref{l:expectations} it
 follows that:
 \begin{align*}
   \E\big[\tr\big((\X^\top\S_{\pi}^\top\I_{[d]}\S_{\pi}\X)^{-1}\big)\
   |\,\neg A\big]
   &=   \frac{k}{d}\cdot\tr\big(
     \E\big[(\X^\top\S_{\pi_{[d]}}^\top\S_{\pi_{[d]}}\X)^{-1}\big]\big)
\leq \frac kd\cdot d\,\tr\big((\X^\top\X)^{-1}\big) = k\phi.
 \end{align*}
 Here the $k/d$ term arises from the different multipliers used for $\S_\pi$ and
 $\S_{\pi_{[d]}}$.
 Thus, we obtain that $\E\big[\|(\S_\pi\X)^\dagger\S_\pi\xib\|^2\,|\,\neg A\big]\leq
 4nk\phi\, \|\xib\|^2$. It remains to show that $\Pr(\neg A)$ is sufficiently
 small to obtain the desired bound. For this, we refer to a standard
matrix concentration result for obtaining subspace embeddings, which
follows from \cite{matrix-tail-bounds}.
\begin{lemma}
Given a full rank $\X\in\R^{n\times d}$, if distribution $q$ is such that
$q_i\geq \beta\,\frac{l_i(\X)}{d}$ for all $i\in[n]$, then an i.i.d.~sample
$\pi_1,\dots,\pi_s\sim q$ for $s\geq C'\frac d\beta\log
\frac d\delta$ with probability at least $1-\delta$ satisfies
\begin{align*}
\frac12\,\X^\top\X\preceq
  \frac1s\sum_{i=1}^s\frac1{q_{\pi_i}}\x_{\pi_i}\x_{\pi_i}^\top\preceq
  \frac32\,\X^\top\X. 
\end{align*}
\end{lemma}
Setting $\beta=\frac14$, $\delta = \frac1{6nk^2}$, $s=k-d$ and $C\geq
16C'$, we conclude that
since $k-d\geq 16C' d\log n\geq C'\frac d\beta\log \frac d\delta$, we have $\Pr(\neg A)
\E\big[\|(\S_\pi\X)^\dagger\S_\pi\xib\|^2\,|\,\neg A\big] \leq \frac \phi k \|\xib\|^2 \leq \frac\epsilon
{16} \|\xib\|^2$. Combining this with \eqref{eq:terms} and \eqref{eq term 1}, 
we complete the proof of the bound of Lemma \ref{l:key}.
\end{proof}
The last term in the bound of Lemma \ref{l:key} can be
controlled when the vector $\xib$ is orthogonal to the columns of
$\X$. The following bound is shown in Appendix \ref{a:mult}.
\begin{lemma}\label{l:bias}
  For $\X$ and $\xib$ s.t.~$\X^\top\xib=\zero$, if
  $\pi\sim\Vol_{q(\alpha)}^k(\X)$ (as in Lemma \ref{l:key}) with
  $\alpha=0.5$, then 
  \begin{align*}
    \big\|\X^\dagger\E[\S_\pi^\top\S_\pi]\xib\big\|^2\leq
    \frac\epsilon 8\|\xib\|^2.
  \end{align*}
\end{lemma}
We put the two lemmas together to complete the proof of our main result.
\begin{proofof}{Theorem}{\ref{t:mse}}
Setting $\wbh=(\S_\pi\X)^\dagger\S_\pi\y$ for a fixed vector
$\y\in\R^n$ with $\pi$ as in Lemma
\ref{l:key} and $\alpha=0.5$, using Lemmas \ref{l:key} and
\ref{l:bias} with $\xib=\y-\X\wols$ we obtain that:
\begin{align*}
  \E\big[\|\wbh-\wols\|^2 \big] \overset{(*)}{=}
  \E\big[\|(\S_\pi\X)^\dagger\S_\pi (\y- \X\wols)\|^2\big]
  \leq \frac\epsilon 8\|\xib\|^2 + 4\cdot\frac\epsilon 8\|\xib\|^2\leq
  \epsilon\cdot\|\xib\|^2,
\end{align*}
where $(*)$ follows because volume sampling ensures that
$\rank(\S_\pi\X)=d$, so $(\S_\pi\X)^\dagger\S_\pi\X=\I$.
\end{proofof}

In Appendix \ref{a:mspe} we prove Theorem \ref{t:mspe} as a reduction from Theorem
\ref{t:mse} by transforming the matrix $\X$ into matrix
$\U=\X(\X^\top\X)^{-\sfrac12}$. This transformation makes MSE equal to
MSPE, while preserving most key properties of least squares estimators.
The random experimental design obtained via this reduction is different
than the one used to bound the mean squared error. While the leverage
scores are preserved during the
transformation from $\X$ to $\U$, the inverse scores change, and in
fact they become equal to the leverage scores, i.e.,
$v_i(\U)=l_i(\U)$, so distribution $q$ is somewhat simpler in this
case. However, it can be shown that the exact experimental design
used for Theorem \ref{t:mse} also satisfies the guarantee from Theorem
\ref{t:mspe}, albeit with slightly different constants.

The logarithmic dependence on $n$ in the sample size $k$ for Theorems
\ref{t:mse} and \ref{t:mspe} comes from our analysis of the expected
error in the tail of the sampling distribution. It is possible that
the dependence on $n$ can be eliminated altogether, even when using
the same distribution. We leave this as an open question for future work.

\acks{MWM would like
  to acknowledge ARO, DARPA, NSF and ONR for providing partial
  support of this work. Also, MWM and MD thank the NSF for
  funding via the NSF TRIPODS program. Part of this work
  was done 
  while MD, KLC and MWM were visiting the Simons Institute for the
  Theory of Computing and while MKW was at UC Santa Cruz, supported by
  NSF grant IIS-1619271.}  

\bibliography{pap}

\appendix
\newpage

\section{Basic properties of the minimax-optimal experimental design}
\label{a:minimax}
We start by formally showing that
the least squares estimator $\wols(\y|\X)=\X^\dagger\y$ is 
the minimum variance unbiased estimator (MVUE) for
$\w^*=\X^\dagger\E[\y]$ w.r.t.~the family $\Fc$ consisting of all
random response vectors $\y$ with finite second moment.
\begin{proofof}{Proposition}{\ref{p:mvue}}
  Since all fixed vectors $\y\in\R^n$ belong to $\Fc$, it follows
  that $\E\big[\wbh(\y)\,|\,\y\big] = \wols(\y|\X)$ and using
  shorthands $\wbh=\wbh(\y)$, $\wols=\wols(\y|\X)$ and
  $\w^*=\X^\dagger\E[\y]$ we have
  \begin{align*}
    \Var\big[\wbh\big]
    &= \E\big[\wbh\wbh^\top\big] - \w^*\w^{*\top}\\
    &=\E_\y \Big[\E_{\wbh}\big[\wbh\wbh^\top-\wols\wols^\top\,|\,\y\big]
      +\wols\wols^\top\Big] - \w^*\w^{*\top}\\
    &=\E\big[(\wbh-\wols)(\wbh-\wols)^\top\big] +
      \E\big[\wols\wols^\top\big] - \w^*\w^{*\top}\succeq \Var\big[\wols\big],
  \end{align*}
  because the first term is a positive semi-definite matrix.
\end{proofof}
In the next lemma, we observe that it suffices to consider fixed
vectors $\y\in\R^n$ to bound $R_k^*(\X)$.
\begin{lemma}\label{l:fixed}
  Given a matrix $\X\in\R^{n\times d}$, suppose that a random design
  $(S,\wbh)$ for all fixed response vectors $\y\in\R^n$ satisfies 
  \begin{align*}
    \E\big[\wbh(\y_S)\big]=\wols\quad\text{and}\quad
    \E\big[\|\wbh(\y_S)-\wols\|^2\big]\leq
    \epsilon\cdot\|\X\wols-\y\|^2,
  \end{align*}
where $\wols=\wols(\y|\X)$. Then, for all random response vectors
$\y\in \Fc$, we have
  \begin{align*}
    \E\big[\wbh(\y_S)\big] = \w^*\quad\text{and}\quad
    \E\big[\|\wbh(\y_S)-\w^*\|^2\big]\leq \E\big[\|\wols-\w^*\|^2\big]+ \epsilon\cdot \E\big[\|\X\w^*-\y\|^2\big].
  \end{align*}
\end{lemma}
\begin{proof}
We first decompose the mean squared error using the unbiasedness of
$\wbh(\y_S)$ as follows:  
\begin{align*}
  \E\big[\|\wbh(\y_S)-\w^*\|^2\big]
  &=\E\big[\|\wbh(\y_S)-\wols\|^2\big] +
    \E\big[\|\wols-\w^*\|^2\big]\\
  &\leq \epsilon\cdot \E\big[\|\X\wols-\y\|^2\big] + \E\big[\|\wols-\w^*\|^2\big].
\end{align*}
Also note that $\E[\wbh(\y_S)] = \E\big[\E[\wbh(\y_S)\,|\,\y]\big]=\E[\wols]=\w^*$.
It remains to bound $\E\big[\|\X\wols-\y\|^2\big]$.
Note that $\wols = \argmin_\w \norm{\X\w - \y}^2$,
so that in particular $\norm{\X\wols - \y}^2 \le \norm{\X\w^* - \y}^2$,
and therefore $\E\big[\|\X\wols-\y\|^2\big] \le \E\big[\|\X\w^*-\y\|^2\big]$,
which concludes the proof.
\end{proof}
\begin{proofof}{Lemma}{\ref{l:reduction}}
Follows immediately from Lemma \ref{l:fixed} and the fact that $\Vc_k(\X)\subseteq\Wc_k(\X)$.
\end{proofof}
We next prove Proposition \ref{p:minimax-bounded}, showing that
the minimax-optimal value $R_k^*(\X)$ given in Definition
\ref{d:minimax} is well-defined for all $d\leq k\leq n$ and
for most $k$ it has matching upper and lower bounds of $\Theta(\phi/\epsilon)$.
\begin{proofof}{Proposition}{\ref{p:minimax-bounded}} 
  
\textit{Part 1.} \ Since $\mathrm{MSE}[\wbh]=\tr\big(\Var[\wbh]\big)$
for any unbiased estimator, Proposition
\ref{p:mvue} immediately implies that $R_k^*(\X)\geq 0$. Next,
let $\pi\sim\Vol_q^k(\X)$ with $q$ chosen as in the proof of Theorem
\ref{t:mse} and consider the volume sampled estimator defined as in
\eqref{eq:design}, i.e. $\wbh=(\S_\pi\X)^\dagger\S_\pi\y$. First, we
can use Lemma \ref{l:reduction} to assume w.l.o.g.~that $\y$ is a
fixed vector in $\R^n$ so that $\w^*=\X^\dagger\y$. Then, as in
the proof of Theorem \ref{t:mse} we use equation \eqref{eq:square-inverse} in Lemma
\ref{l:expectations}:
\begin{align*}
  \E\big[\|\wbh-\w^*\|^2\big]&\leq
  4n\|\X\w^*-\y\|^2\,\E\big[\tr\big((\X^\top\S_\pi^\top\S_\pi\X)^{-1}\big)\big]\\
  &\leq 4n\|\X\w^*-\y\|^2\frac{k}{k-d+1}\tr\big((\X^\top\X)^{-1}\big).
\end{align*}
This implies that $R_k^*(\X)\leq 4n\frac{k}{k-d+1}\,\phi<\infty$, 
where unlike in the proof of Theorem \ref{t:mse} we did not need to
assume any lower bound on $k$ other than that $k\geq d$.

\textit{Part 2.} \
Theorem \ref{t:mse} implies that there is a constant $C>0$ such that
for any $\X$, $\epsilon>0$ and $k\geq C\cdot (d\log n+\phi/\epsilon)$ there
is a random experimental design $(S,\wbh)$ of size at most $k$ which
demonstrates an upper bound on the minimax-optimal value, i.e., that
$R_k^*(\X)\leq \epsilon$. Now, suppose that $k\geq 2C\cdot d\log n$ and
let $\epsilon=2C\cdot\phi/k$. Then,
$
  k\geq 2C\cdot \max\{d\log n,\,\phi/\epsilon\}\geq C\cdot(d\log n + \phi/\epsilon),
$
which means that for any $\X$ we have $R_k^*(\X)\leq \epsilon=2C\cdot\phi/k$.

\textit{Part 3.} \ This result is based on the following lower bound for
classical $A$-optimal design.
\begin{theorem}[\citeauthor{avron-boutsidis13},
  \citeyear{avron-boutsidis13}, Theorem 4.5]\label{t:lower}
  For any $\alpha>0$, $n$, $d$ such that $n>2d$ and
  $\mathrm{mod}(n,d)=0$ there is a full rank $n\times d$ matrix $\X$
  such that for any subset $S\subseteq[n]$
  with $\mathrm{rank}(\X_S)=d$ and $|S|=k$, we have
  \begin{align*}
    \tr\big((\X_S^\top\X_S)^{-1}\big)\geq
    \bigg(\frac{n-k}{k+\alpha^2}+1-\frac kd\bigg)\cdot\tr\big((\X^\top\X)^{-1}\big).
  \end{align*}
\end{theorem}
We remark that $\X$ in the theorem depends on $\alpha$.
Note that the condition $\mathrm{mod}(n,d)=0$ can be eliminated by
padding matrix $\X$ with appropriate number of $\zero$ rows and
replacing $n$ in the bound with $n-d$. Let $\X$ be the
matrix from Theorem \ref{t:lower} (padded if necessary, with $\alpha$
chosen later) and let
  $\Fc_N(\X)$ be the family of random response vectors $\y=\X\w^*+\xib$
  such that $\xib\sim\Nc(\zero,\I)$ and $\w^*\in\R^d$. 
Also let $\Fc_B(\X)$ be the Bayesian counterpart, where $\y=\X\w+\xib$
for independent Gaussian random vectors
$\w\sim\Nc(\zero,\sigma_\w^2\I)$ and $\xib\sim\Nc(\zero,\I)$, with any
$\sigma_\w^2>0$. Given any $\y\in\Fc_B(\X)$ with prior variance
$\sigma_\w^2$, the following bound is known for any (possibly biased)
estimator $\wbh(\y)$ of $\w$, called the minimum mean squared error
bound \citep[MMSE; see][]{gray-davisson-stats-book}: 
\begin{align*}
  \E\big[\|\wbh(\y)-\w\|^2\big] \geq \tr\big((\X^\top\X+(1/\sigma_\w^2)\I)^{-1}\big).
\end{align*}
An analogous lower bound holds when applied to the same regression model
restricted to a fixed subset $S_\circ\subseteq[n]$,
i.e.~such that $\y_{S_\circ}=\X_{S_\circ}\w+\xib_{S_\circ}$, and any estimator
$\wbh(\y_{S_\circ})$: 
\begin{align}
  \E\big[\|\wbh(\y_{S_\circ})-\w\|^2\big]\geq \tr\big((\X_{S_\circ}^\top\X_{S_\circ}+(1/\sigma_\w^2)\I)^{-1}\big).\label{eq:bayes}
\end{align}
We use this to give a lower bound for the MSE of any random experimental
design $(S,\wbh)\in\Wc_k(\X)$ following Definition \ref{d:rand} for
the (non-Bayesian) response model $\Fc_N(\X)$:
\begin{align*}
  \max_{\y\in\Fc_N(\X)}\MSE{\wbh(\y_S)}
  &= \max_{\underset{\y=\X\w+\xib}{\y\in\Fc_B(\X)}}
\max_{\w^*\in\R^d}\E\big[\|\wbh(\y_S)-\w\|^2\,|\,\w\!=\!\w^*\big]\\[-2mm]
    &\overset{(a)}{\geq} \max_{\underset{\y=\X\w+\xib}{\y\in\Fc_B(\X)}}\E_\w\Big[\E\big[\|\wbh(\y_S)-\w\|^2\,|\,\w\big]\Big]\\
    &= \max_{\underset{\y=\X\w+\xib}{\y\in\Fc_B(\X)}} \E\big[\|\wbh(\y_S)-\w\|^2\big]\\
  &=\max_{\underset{\y=\X\w+\xib}{\y\in\Fc_B(\X)}}\E_S\Big[\E\big[\|\wbh(\y_S)-\w\|^2\,|\,S\big]\Big]\\[-2mm]
  &\overset{(b)}{\geq} \lim_{\sigma_\w\rightarrow\infty}
    \E_S\Big[\tr\big((\X_S^\top\X_S+(1/\sigma_\w^2)\I)^{-1}\big)\Big]\\
  &
 \geq \min_{\underset{\rank(\X_S)=d}{S:|S|\leq k}}\tr\big((\X_S^\top\X_S)^{-1}\big),
\end{align*}
where $(a)$ follows because expectation $\E_\w$ is always upper
bounded by the maximum over all possible values of $\w$, and $(b)$
follows from \eqref{eq:bayes} applied to the conditional expectation
for a fixed $S$ and the fact that the trace is monotonically increasing with $\sigma_\w$.
We now lower bound $R_k^*(\X)$ by the minimax value based on the
response family $\Fc_N(\X)$: 
  \begin{align*}
    R_k^*(\X)
    &\geq
\min_{(S,\wbh)\in\Wc_k(\X)}\
      \max_{\y\in\Fc_N(\X)}\frac{\MSE{\wbh(\y_S)}-
      \MSE{\wols(\y|\X)}}{\E_\y\big[\|\xib_{\y|\X}\|^2\big]}\\
   &\geq\min_{\underset{\rank(\X_S)=d}{S:|S|\leq k}}\frac1n\Big(\tr\big((\X_S^\top\X_S)^{-1}\big)-\tr\big((\X^\top\X)^{-1}\big)\Big)\\ 
    \text{(Theorem \ref{t:lower})}\quad&\geq
\frac1n\bigg(\frac{n-d-k}{k+\alpha^2}-\frac kd\bigg)\cdot
                                         \tr\big((\X^\top\X)^{-1}\big)\\
&    \geq  \bigg(\frac k{k+\alpha^2} - \frac{3k^2}{nd}\bigg)\cdot
      \frac\phi k\ \geq\  (1-\epsilon)\cdot\frac\phi k,
  \end{align*}
because $k^2<\epsilon nd/3$ and we can choose $\alpha$ small enough
so that $\frac{\alpha^2}{k+\alpha^2}+\frac{3k^2}{nd} \leq \epsilon$.
\end{proofof}

\section{Omitted proofs from Section \ref{s:proofs}}
\label{a:mult}

\begin{proofof}{Lemma}{\ref{l:mult}}
Recall that $\v_\pi = \A^\top\S_\pi^\top\S_\pi\b$.  We rewrite the expectation as follows:
\begin{align*}
  \E\big[\|\v_\pi-\E[\v_\pi]\|^2\big]
  &=\b^\top\E\Big[
    \big(\S_{\pi}^\top\S_{\pi}\!-\!\E[\S_{\pi}^\top\S_{\pi}]\big)\A\A^\top\!
    \big(\S_{\pi}^\top\S_{\pi}\!-\! \E[\S_{\pi}^\top\S_{\pi}]\big)\Big]\ \b\\
  &\le \bigg\|\underbrace{
    \begin{bmatrix}\cov\big[\frac{s_i}{kq_i},\frac{s_j}{kq_j}\big]\a_i^\top\a_j\end{bmatrix}_{n\times n}}_{\M}\bigg\|\cdot\|\b\|^2,
\end{align*}
where $s_i=|\{t\,:\,\pi_t=i\}|$.
Note that $\M$ is the Hadamard product of two PSD matrices, and therefore also PSD
by the Schur product theorem.
Next, we use two formulas
shown by \cite{leveraged-volume-sampling} for rescaled volume sampling:
\begin{align}
  \E[s_i] &=(k-d)\,q_i+l_i,\label{eq:marginal}\\
  \cov(s_i,s_j)&=\one_{i=j}\E[s_i] - (k-d)q_iq_j - l_{ij}^2,\nonumber
\end{align}
where $l_i=\a_i^\top(\A^\top\A)^{-1}\a_i$ is the $i$th leverage score
of $\A$ and $l_{ij}=\a_i^\top(\A^\top\A)^{-1}\a_j$ is the $(i,j)$th 
cross-leverage score. Using the above we get:
\begin{align*}
  \M = \text{diag}\bigg(\begin{bmatrix}\frac{\|\a_i\|^2}{kq_i}\frac{(k-d)q_i+l_i}{kq_i}\end{bmatrix}_{n\times 1}\bigg) -
  \frac{k-d}{k^2}\A\A^\top
      - \begin{bmatrix}\frac{l_{ij}^2}{k^2q_iq_j}\a_i^\top\a_j\end{bmatrix}_{n\times n}.
\end{align*}
Note that the second term is a PSD matrix being subtracted from
$\M$.
Similarly, the last term is also
subtracting a PSD matrix. To see this, note that the matrix
formed by the cross-leverage scores is $[l_{ij}]_{n\times n}=\A(\A^\top\A)^{-1}\A^\top$,
so it is PSD.
Therefore, we can write the last term as a Hadamard product of
PSD matrices and apply Schur product theorem. Thus it remains to bound
the diagonal term: 
\begin{align*}
  \frac{\|\a_i\|^2}{kq_i}\,\frac{(k-d)q_i+l_i}{kq_i}\leq
  \frac{\|\A\|_F^2}{\beta k} \, \bigg(\frac{k-d}{k} +
  \frac{d}{k\beta}\bigg)\leq  \frac{\|\A\|_F^2}{\beta^2 k}.
\end{align*}
Since the spectral norm of the diagonal term is bounded by $\frac{\|\A\|_F^2}{\beta^2 k}$,
and subtracting the two PSD terms leaves $\M$ a PSD matrix, we have
$\norm{\M} \le \frac{\|\A\|_F^2}{\beta^2 k}$, and the result follows.
\end{proofof}

\begin{proofof}{Lemma}{\ref{l:bias}}
Using the marginal expectation formula \eqref{eq:marginal} for
volume sampling we have:
\begin{align*}
  \E[\S_\pi^\top\S_\pi] =
  \text{diag}\bigg(\begin{bmatrix}\frac{(k-d)q_i+l_i(\X)}{kq_i}
  \end{bmatrix}_{n\times
  1}\bigg)
  =   \text{diag}\bigg(\begin{bmatrix}\frac{q_i(\frac
  {k-d}k\alpha)}{q_i(\alpha)}\end{bmatrix}_{n\times 1}\bigg)\defeq \D_\alpha,
\end{align*}
where 
  $q(\alpha)= \alpha
  \big(0.5\cdot p^{\mathrm{uni}}+0.5\cdot p^{\mathrm{inv}}\big)
  +(1\!-\!\alpha)p^{\mathrm{lev}}$. Now, let $\beta =
  \alpha\cdot\frac{k}{k-d}$ and $\pi'\sim \Vol_{q(\beta)}^k(\X)$. 
Using expectation formula \eqref{eq:unbiased} of Lemma
\ref{l:expectations} and the fact that $\X^\dagger\xib=\zero$, we have:
\begin{align*}
  \E\big[(\S_{\pi'}\X)^\dagger\S_{\pi'}\,\D_\beta^{-1}\xib]
  &= \X^\dagger \D_{\beta}^{-1}\xib \\
&=\X^\dagger(\D_{\beta}^{-1}-\I)\xib\\
  &=\X^\dagger \text{diag}\bigg(\frac{q_i(\beta) -
    q_i(\alpha)}{q_i(\alpha)}\bigg)\xib\\
  &=-\frac{k}{k-d}\X^\dagger(\D_{\alpha}-\I)\xib\\
  &=-\frac{k}{k-d}\X^\dagger\D_\alpha\xib.
\end{align*}
We can assume $k\geq 3d$, adjusting $C$ of Lemma~\ref{l:key}.
With $\alpha=0.5$, this implies $\beta\in[0.5,0.75]$, so
applying Lemma \ref{l:key} to sequence $\pi'$ and vector
$\D_\beta^{-1}\xib$ combined with Jensen's inequality we obtain:
\begin{align*}
  \|\X^\dagger\D_\alpha\xib\|^2
  &=\Big(\frac{k-d}{k}\Big)^2\cdot
\big\|\E[(\S_{\pi'}\X)^\dagger\S_{\pi'}\D_\beta^{-1}\xib]\big\|^2\\ 
  \text{(Jensen's inequality)}\quad
  &\leq
    \Big(\frac{k-d}{k}\Big)^2\cdot\E\big[\|(\S_{\pi'}\X)^\dagger\S_{\pi'}\D_\beta^{-1}\xib\|^2\big]\\
  \text{(Lemma \ref{l:key})}\quad
  &\leq \Big(\frac{k-d}{k}\Big)^2\cdot \Big(\frac\epsilon8\,\|\D_\beta^{-1}\xib\|^2 +
    4\underbrace{\|\X^\dagger\D_\beta\D_\beta^{-1}\xib\|^2}_{0}\Big)\\[-3mm]
  &\leq \frac\epsilon8\cdot\|\xib\|^2,
\end{align*}
because $\|\D_\beta^{-1}\|\leq \frac{k}{k-d}$, which concludes the proof.
\end{proofof}

\section{Proof of Theorem \ref{t:mspe}}
\label{a:mspe}
  The key idea in the proof is a standard transformation of the
  data matrix $\X$ which has the property that it preserves the
  predictions of the least squares estimator, while transforming the
  actual estimator in such a way that the mean squared error becomes
  equal to the mean squared prediction error. Specifically, consider
  matrix $\U=\X(\X^\top\X)^{-\sfrac12}$. This matrix has the property
  that $\U^\top\U=\I$ so $\tr((\U^\top\U)^{-1})=d$ and $\U\U^\dagger=\U\U^\top=\X\X^\dagger$. Replacing all least squares
  estimators for this new matrix we have $\v^* = \U^\top\E[\y]$,
  $\v_{\mathrm{LS}}=\U^\top \y$ and $\vbh=(\S_\pi\U)^\dagger\S_\pi
  \y$. Note that we have $\U\v^*=\X\X^\dagger\E[\y]=\X\w^*$ and
  similarly $\U\v_{\mathrm{LS}}=\X\wols$. A simple calculation also
  reveals that $\U\vbh=\X\wbh$ for
  $\wbh=(\S_\pi\X)^\dagger\S_\pi\y$.
   Suppose that $\pi\sim\Vol_q^k(\U)$
  is produced as in
  the proof of Theorem \ref{t:mse} when applied 
  to matrix $\U$. Then: 
  \begin{align*}
    \E\big[\|\X(\wbh-\w^*)\|^2\big]
    &= \E\big[\|\U(\vbh-\v^*)\|^2\big]\\
&= \E\big[\|\vbh-\v^*\|^2\big]\\
    \text{(Theorem \ref{t:mse})}\quad
    &\leq \E\big[\|\v_{\mathrm{LS}}-\v^*\|^2\big] + \epsilon\cdot \E\big[\|\U\v^*-\y\|^2\big]\\
      &= \E\big[\|\X(\wols-\w^*)\|^2\big] + \epsilon\cdot \E\big[\|\X\w^*-\y\|^2\big].
  \end{align*}

\end{document}